\documentclass[10pt]{article} 

\usepackage[preprint]{rlj} 

\usepackage{amssymb}            
\usepackage{mathtools}          
\usepackage{mathrsfs}           
\usepackage{graphicx}           
\usepackage{subcaption}         
\usepackage[space]{grffile}     
\usepackage{url}                
\usepackage{lipsum}             

\usepackage{mathtools}
\usepackage{amsmath}
\usepackage{amssymb}
\usepackage{amsthm}
\usepackage{graphicx}
\usepackage{algorithm}
\usepackage[noend]{algpseudocode}  

\newtheorem{theorem}{Theorem} 
\newtheorem{theorem*}{Theorem}{} 
\newtheorem{lemma}{Lemma} 

\newtcolorbox{thmbox}{
    colback=black!5!white,    
    colframe=white!50!black,  
    boxrule=1pt,             
    arc=4pt,                 
    notitle,                 
    top=6pt, bottom=6pt,     
    left=4pt, right=4pt      
}
\tcbuselibrary{theorems} 

\tcbset{
    softbluebox/.style={
        colback=blue!3!white,    
        colframe=blue!20!white,  
        boxrule=0.5pt,           
        arc=3pt,                 
        top=4pt, bottom=4pt,     
        left=8pt, right=8pt      
    }
}
\newtcolorbox{algobox}[1][]{
    colback=gray!5!white,
    colbacktitle=gray!5!white, 
    colframe=gray!70!black, 
    boxrule=0pt,            
    toprule=0pt,            
    bottomrule=0pt,         
    rightrule=0pt,          
    leftrule=4pt,           
    titlerule=0pt,          
    arc=0pt,
    outer arc=0pt,
    coltitle=black,
    fonttitle=\bfseries,
    top=4pt,
    bottom=4pt,
    left=6pt,
    right=6pt,
    title={#1}
}

\usepackage{color}

\title{Maximum Entropy Exploration \\Without the Rollouts}
\setrunningtitle{Maximum Entropy Exploration Without Rollouts}

\author{Jacob Adamczyk\textsuperscript{1,2}, Adam Kamoski\textsuperscript{1}, Rahul V. Kulkarni\textsuperscript{1,2}
}

\emails{\{jacob.adamczyk001, adam.kamoski001, rahul.kulkarni\}@umb.edu}

\affiliations{
$^{1}$\textbf{Department of Physics, University of Massachusetts Boston}\\
$^{2}$\textbf{NSF Institute for Artificial Intelligence and Fundamental Interactions}

}

\newcommand{\abs}{%
Efficient exploration remains a central challenge in reinforcement learning, serving as a useful pretraining objective for data collection, particularly when an external reward function is unavailable. A principled formulation of the exploration problem is to find policies that maximize the entropy of their induced steady-state visitation distribution, thereby encouraging uniform long-run coverage of the state space. Many existing exploration approaches require estimating state visitation frequencies through repeated on-policy rollouts, which can be computationally expensive. In this work, we instead consider an intrinsic average-reward formulation in which the reward is derived from the visitation distribution itself, so that the optimal policy maximizes steady-state entropy. An entropy-regularized version of this objective admits a spectral characterization: the relevant stationary distributions can be computed from the dominant eigenvectors of a problem-dependent transition matrix. This insight leads to a novel algorithm for solving the maximum entropy exploration problem, EVE (EigenVector-based Exploration), which avoids explicit rollouts and distribution estimation, instead computing the solution through iterative updates, similar to a value-based approach. To address the original unregularized objective, we employ a posterior-policy iteration (PPI) approach, which monotonically improves the entropy and converges in value. We prove convergence of EVE under standard assumptions and demonstrate empirically that it efficiently produces policies with high steady-state entropy, achieving competitive exploration performance relative to rollout-based baselines in deterministic grid-world environments.
}

\summary{\abs
}

\begin{document}

\maketitle  

\begin{abstract}
\abs
\end{abstract}

\section{Introduction}
Efficient exploration in reinforcement learning (RL) remains a long-standing open problem. A rich subfield has emerged to tackle this problem, with new research questions fueling new ideas such as entropy- and information-seeking agents~\citep{tiomkin2017unified}, count-based systems~\citep{bellemare2016unifying,ostrovski2017count, zhou2020neural, lobel2023flipping}, curiosity-driven exploration~\citep{pathak2017curiosity} and random network distillation~\citep{burda2018exploration}. However, novel approaches and insights are still needed to address a fundamental question of exploration: \textit{How can an agent learn to obtain uniform coverage in a new environment?}

An important class of approaches for exploration in RL focuses on objectives derived from visitation frequencies. Although efficient algorithms have been formulated and developed using these objectives~\citep{hazan2019provably}, such approaches are inherently computationally demanding.
By design, the entropy objective in these approaches is 
dependent on the stationary visitation distribution induced by the policy itself. Evaluating this distribution typically requires estimating visitation frequencies through repeated rollouts. Because the objective is defined in terms of the policy-induced distribution, improving the policy requires repeatedly estimating its corresponding distribution, creating a circular dependency between policy updates and visitation estimates. In practice, this often necessitates on-policy sampling and can make the optimization problem computationally expensive.

In this work, we address the maximum entropy exploration problem from a new perspective by constructing an update equation inherently consistent with the desired visitation distribution, without requiring rollouts. We arrive at an algorithm which uses a new (reward-free) two-step temporal difference algorithm to learn the optimal policy directly.

Recent research provides a framework for addressing the challenges noted above. For systems with deterministic dynamics, Bayesian inference can be used to solve the optimal control problem in entropy-regularized RL~\citep{levine2018reinforcement}. In the average-reward setting, analytical expressions for the optimal policy and the associated state-action visitation distributions have been derived~\citep{PRR}. Our work leverages this distinctive spectral viewpoint, in which optimal policies and visitation distributions are characterized through the eigenvectors of a tilted transition operator. 

Building on the analytical results from~\cite{PRR}, we obtain expressions for key quantities in exploration, including distributions in state-action space, the intrinsic reward function, and their relationship to the environment's transition dynamics. These analytic results lead to a new update equation that can be efficiently computed, by balancing forward and backward probability flows. The resulting algorithm leads to policies that maximize the entropy of the induced state-action visitation distribution.

We consider the ``reward-free'' setting: there are no extrinsic rewards and the agent's objective is to explore the environment uniformly.
Our proposed algorithm, EVE (``EigenVector-based Exploration'') uses the dominant eigenvectors of the tilted transition matrix to determine the unique entropy-maximizing policy, solving the maximum entropy exploration problem. 
When an extrinsic reward function is provided, downstream tasks can then leverage the data collected by EVE to solve complex tasks with potentially sparse rewards. Alternatively, one might use EVE in the online setting to define the behavior policy for active exploration.
EVE shows that maximum-entropy exploration can be solved through a single fixed-point problem, allowing exploratory policies to be computed directly from the transition dynamics without estimating visitation frequencies through rollouts.

In the following section, we provide necessary background material in reinforcement learning, highlighting the entropy-regularized average-reward objective before presenting our framework for calculating optimal policies. We then describe the corresponding algorithm, its convergence properties, and provide a validating experiment. We end the paper by connecting to related work and discussing the results.

\section{Background}\label{sec:background}
Consider an environment with state space $\mathcal{S}$ and action space $\mathcal{A}$, and an agent interacting with the environment through a (generally stochastic) policy $\pi(a|s)$. The environment's transition dynamics $p(s'|s,a)$ describe the next state $s'$ based on the current state $s$ and action taken, $a$. In the following, we focus on the discrete state-action case for ease of exposition. The typical RL objective is to optimize the accumulated reward, $r(s,a)$. That is, to solve the following optimization problem:
\begin{equation}\label{eq:discounted}%
    J^*=\max_\pi \mathbb{E}_{\tau \sim{p, \pi, \mu}}\sum_{t=0}^\infty \gamma^t r(s_t,a_t),%
\end{equation}
where trajectories are sampled starting from an initial distribution $\mu$ and following dynamics $p$ and policy $\pi$ thereafter. Rewards are exponentially discounted over time with a discount factor $\gamma \in [0,1)$.
Much prior work in reinforcement learning considers the discounted objective in Equation~\eqref{eq:discounted}; however, we argue that this choice is not suitable for the exploration problem. A discount factor inherently introduces a finite timescale for the agent's decision-making capacity: $T\sim{(1-\gamma)^{-1}}$. To solve the exploration problem, an agent should visit many states, even those beyond this artificially chosen temporal horizon. To this end, we adopt an average-reward approach to the exploration problem. The average-reward RL objective has received growing interest in recent years~\citep{mahadevan1996average, zhang2021average, wan2021learning, ARTRPO, APO, hisaki2024rvi, adamczyk2025average, rojas2026differential} due to its ability to solve long-horizon, continuing tasks without a discount factor. 
In the maximum entropy exploration problem, we argue that discounting may not always be appropriate: Using a discounted objective will result in entropy maximization of the \textit{discounted} occupancy measure, whereas we may instead be interested in the undiscounted steady-state distribution, which cannot be easily computed with discounted approaches~\citep{naik2019discounted}.
In the present case, we therefore replace the discounted objective in Equation~\eqref{eq:discounted} by the \textit{average reward} objective, defined by:
\begin{equation}\label{eq:avg-rwd-objective}
    \rho^*=\max_\pi \lim_{T\to \infty} \frac{1}{T} \mathbb{E}_{\tau \sim{p, \pi, \mu}} \sum_{t=1}^{T} r(s_t,a_t) = \max_\pi\sum_{s \in \mathcal{S}} \sum_{a \in \mathcal{A}} {d_{p,\pi}}(s,a) r(s,a),
\end{equation}
where $\rho^*$ is the optimal ``reward-rate''. Assuming it exists (cf. the standard well-mixing assumptions, as in~\citep{wan2021learning}), the reward-rate can be expressed as an expectation of rewards under the optimal policy's steady-state distribution (right-hand side of Equation~\eqref{eq:avg-rwd-objective}).

In this work, we are interested in solving the maximum state-action entropy problem. In this setting, ``optimal exploration'' corresponds to finding a policy that achieves a maximally uniform distribution over state-action space. Analogous to the formulation in~\citep{hazan2019provably}, we study the following objective:
\begin{equation}
\tcbhighmath[softbluebox]{
    \max_\pi \mathbb{H}(d_{p,\pi}) = \max_\pi\left( -\sum_{s \in \mathcal{S}} \sum_{a \in \mathcal{A}} {d_{p,\pi}}(s,a) \log d_{p,\pi}(s,a)\right), \label{eq:entropy-objective}
    }
\end{equation}

where $d_{p,\pi}$ represents the stationary distribution under the transition dynamics $p$ and control policy $\pi$, obtained in the long-time limit. In other words, $d_{p,\pi}$ is the non-trivial eigenvector of the state-action transition matrix $p(s'|s,a)\pi(a'|s')$. Throughout, we will assume the transition dynamics is deterministic, aperiodic, and irreducible.

One approach to address this optimization problem is via reinforcement learning: By comparing Equation~\eqref{eq:avg-rwd-objective} to Equation~\eqref{eq:entropy-objective}, we see that one can assign the reward $r(s,a)=-\log {d_{p,\pi}}(s,a)$, to be consistent with the maximum entropy objective. 
With this identification, optimization of the average-reward objective solves this problem of maximum entropy exploration using RL.

To leverage closed-form results in the average-reward setting, we include an entropy-regularization term: $\beta^{-1}D_{\textrm{KL}}(\pi \,\|\, \pi_0)$, where $\pi_0$ is a chosen reference or prior policy (often taken to be uniform, in the case of ``MaxEnt RL''~\citep{SAC1}), and $\beta$ is a real-valued ``inverse temperature'' ($\alpha^{-1}$ in e.g.~\citep{SAC2}) responsible for regulating the importance of the Kullback-Leibler divergence, $D_{\textrm{KL}}$ relative to the accrued rewards. 
The objective for the entropy-regularized average-reward problem can be expressed as:
\begin{equation}
    \theta^* = \max_\pi \lim_{T\to \infty} \frac{1}{T} \mathbb{E}_{\tau \sim{p, \pi, \mu}} \sum_{t=1}^{T} \left( r(s_t,a_t) - \beta^{-1} \log\frac{\pi(a_t|s_t)}{\pi_0(a_t|s_t)}\right). \label{eq:theta-eqn}
\end{equation}
Here, $\theta^*$ takes the role of the reward-rate $\rho^*$, but includes the entropic cost. 
A finite temperature ${\beta < \infty}$ can be used throughout, treating the regularized objective as a surrogate for the un-regularized solution of interest, corresponding to $\beta \to \infty$. In the following sections, we discuss two methods for solving the un-regularized problem.

Recent work by~\cite{PRR} provides a solution to the entropy-regularized average reward problem through the construction of a ``tilted matrix''. We summarize the main results of their framework below. The key to solving the entropy-regularized average reward problem is the tilted matrix, defined as:
\begin{equation}
    \widetilde{P}(s',a'|s,a)=p(s'|s,a) \pi_0(a'|s') e^{\beta r(s,a)}, \label{eq:ptilda}
\end{equation}
which combines the transition dynamics, prior policy, and reward function. The tilted matrix is sub-stochastic (column sums are less than one) and satisfies the Perron-Frobenius theorem. The corresponding dominant eigenvalue of $\widetilde{P}$ is $\lambda=\exp{\beta\theta^*}$, i.e., the exponential of the entropy-regularized average reward-rate, $\theta^*$. These results can be understood as an extension of the linearly solvable framework of~\cite{todorov2006linearly}.
Furthermore, it was shown that the corresponding left eigenvector, satisfying $u^T \widetilde{P}=\lambda u^T$, encodes the optimal policy: $\pi^*(a|s) \propto \pi_0(a|s) u(s,a)$. The distinct, non-trivial right eigenvector $v$ satisfying $\widetilde{P}v=\lambda v$ represents a ``quasi-stationary distribution''~\citep{Mlard2012}. These left and right eigenvectors are related to the backward and forward messages in the Bayesian interpretation, respectively, as discussed by~\cite{levine2018reinforcement, argenis-thesis}. The Perron-Frobenius theorem implies that both vectors $u$ and $v$ are unique and strictly positive. Importantly, the steady-state distribution induced by the optimal policy $\pi^*$ and transition dynamics $p$ can be expressed through the eigenvectors' (Hadamard) product:
\begin{equation}
    d_{p,\pi^*}(s,a) = u(s,a)v(s,a).\label{eq:steady-state}
\end{equation}
This decomposition is similar to that given by~\cite{jain2023maximum}. However, our approach is Markovian, in the average-reward framework, and as shown below, yields a new update equation based on temporal differences.
In the following, we discuss how the ideas from this section can be combined to unlock new insights and algorithms for the maximum-entropy exploration problem.

\section{Results}
We begin with the observation that, rather than using rollouts to estimate the policy-induced entropy, the steady-state distribution can be derived from the eigenvectors of the tilted matrix, attainable from off-policy methods.
Given the linear-algebraic framework discussed above, we can rewrite the intrinsic reward function which optimizes the average entropy-rate of the agent.

Rather than using discounted occupancy measures, this framework allows the average entropy to be optimized directly. 
Comparing with Equation~\eqref{eq:theta-eqn} we see that by choosing the reward function to be of the form 
\begin{equation}
    r(s,a) = -\log u(s,a) v(s,a),\label{eq:reward-defn}
\end{equation}

the optimal reward-rate is exactly the desired entropy-rate. The right eigenvector $v$ can be computed from Equation~\eqref{eq:self-consistent} at any point to estimate the current policy's distribution, and its entropy without any rollouts. This approach is therefore off-policy: data is only needed from the prior policy, not the learned policy. This is in stark contrast to rollout-based methods that require extensive sampling from the online policy.
However, there are two caveats: (1) this is still an implicit solution, where the left and right eigenvectors $u, v$ of the new tilted matrix must be self-consistent with the vectors used to define the reward function; and (2) an exact correspondence to the maximum entropy distribution only holds for the un-regularized problem, where $\beta\to\infty$. 
We address these issues in the following sections.

\subsection{Self-Consistent Solution}
To address the first problem (the self-referential nature of eigenvectors and reward functions), we must find left and right eigenvectors $u,v$ which consistently satisfy the corresponding eigenvector equations and the reward function definition introduced in Section~\ref{sec:background}. After inserting the definition for the reward function (Equation~\eqref{eq:reward-defn}) into the eigenvector equations, we have for $\beta=1$, (see Appendix~\ref{app:proofs} for further details):
\begin{align}
    \sum_{s',a'} u(s',a')P^{\pi_0}{(s',a'|s,a)} &= \lambda u^2(s,a) v(s,a)\ ,\\
    \sum_{s,a} P^{\pi_0}{(s',a'|s,a)}\frac{1}{u(s,a)}&=\lambda v(s',a') \ , \label{eq:self-consistent}
\end{align}
where $P^{\pi_0}{(s',a'|s,a)}=p(s'|s,a)\pi_0(a'|s')$ is the joint state-action transition dynamics and ${0<\lambda<1}$ is the dominant (Perron-Frobenius) eigenvalue, which encodes the optimal entropy-regularized reward rate, $\theta^*$, as discussed above. Note that this optimal reward-rate $\theta^*$ includes a contribution from the average entropy achieved by the policy's steady-state distribution (our original objective) \textit{and} the entropy of the optimal policy itself. In the next subsection, we discuss how to mitigate the latter term's effect.

We now note that Equation~\eqref{eq:self-consistent} allows one to eliminate the right eigenvector $v$. From an algorithmic standpoint, this reduces the required memory and simplifies the problem to estimating a single function. The relationship derived from the above eigenvector equations now provides our novel fixed point iteration scheme, which for general $\beta \neq 1$, reads: 
\begin{equation}\label{eq:u-update}
    u(s,a)\leftarrow\mathcal{T} u(s,a) \doteq \left(\frac{\left(\sum_{s',a'} u(s',a') P^{\pi_0}(s',a'|s,a)\right)^{\frac{1}{\beta}}}{\sum_{\bar{s},\bar{a}} P^{\pi_0}(s,a|\bar{s},\bar{a}) u(\bar{s},\bar{a})^{-\frac{1}{\beta}} (\sum_{s,a} u(s,a) P^{\pi_0}(s,a|\bar{s},\bar{a}))^{\frac{1-\beta}{\beta}}}\right)^{\frac{\beta}{1+\beta}}
\end{equation}
where $(\bar{s},\bar{a})$, $(s,a)$, $(s',a')$ denote successive state-action pairs.
This update equation combines information from the future (numerator) and the past (denominator). To gain a better intuition, we can rewrite this equation in log-space (motivated by the connection to the value function, since {$q(s,a)=\beta^{-1}\log u(s,a)$}), and highlight the simpler case of $\beta=1$, finding:
\begin{equation}\label{eq:q-update}
\tcbhighmath[softbluebox]{
    q(s,a)= \frac{1}{2} \log \left(\mathbb{E}_{a'\sim{\pi_0}} e^{q(s',a')}\right)  - \frac{1}{2}  \log\left(\sum_{\bar{s},\bar{a}}P(s,a|\bar{s},\bar{a})e^{-q(\bar{s},\bar{a})}\right).
    }
\end{equation}
From here, we see that $q$ satisfies a certain ``soft flow'' equation: $q$ balances between the flows out of a state (first term) and the flows into the state (second term). The first term represents a soft maximum over the next state-action pairs while the second term represents a soft minimum over the preceding state-action pairs. 

The derived update scheme simplifies the expensive iteration (computing optimal policy, rollouts, new reward function, etc.) into a single fixed-point equation. Interestingly, even without a discount factor, 
successive iterates of EVE maintain their magnitude. This stability is a result of our convergence proof and the homogeneity of the updates (cf. Appendix~\ref{app:proofs}).
Importantly, since we do not rely on discounting, this flow equally weights incoming (previous) and outgoing (future) transitions, ensuring that temporally distant states are visited with uniform probability, and not artificially weighted by the discounted measure. Upon convergence, EVE therefore provides the policy which maximizes entropy of the \textit{undiscounted} distribution over state-action pairs.

The following theorem guarantees that, despite lacking an obvious contraction factor, the fixed point iteration in Equation~\eqref{eq:u-update} is indeed a contraction mapping and thus converges to a unique function, for all $\beta\geq 1$. Below, let $\kappa(A)$ denote the projective diameter of a positive matrix, $A$. We restrict the operation of $\mathcal{T}$ to the cone $\mathbb{R}^{|\mathcal{S}||\mathcal{A}|}_{>0}$ (the positive orthant).
\newcommand{\EVEConvergence}{%
Let the dynamics $p(s'|s,a)$ be irreducible and aperiodic, and denote by $m$ the index of primitivity for the Markov chain over state-actions induced by $\pi_0$. The mapping $u\leftarrow \mathcal{T}(u)$ given by Equation~\eqref{eq:u-update} is a contraction under the projective metric, and converges linearly to a unique fixed point with effective rate $\kappa((P^{\pi_0})^m)^{1/m}$ when $\beta\geq 1$.
}
\begin{thmbox}
\begin{theorem}[Convergence of EVE]
    \EVEConvergence
\end{theorem}\label{thm:eve-convergence}
\end{thmbox}
The proof of this statement is given in Appendix~\ref{app:proofs} and follows from the properties of Hilbert's projective metric. Intuitively, this can be seen as a form of non-linear Perron-Frobenius theory~\citep{lemmens2012nonlinear}. The projective metric is a logarithmic form of the more well-known \textit{span semi-norm} common in the average-reward literature~\citep{zhang2021finite}.

\subsection{The Un-regularized Objective}
The entropy-regularized objective maintains an additional cost for deviating from a prior policy $\pi_0$. This regularization technique has proven beneficial in some contexts, like MaxEnt RL where $\pi_0$ is uniform~\citep{SAC1}, offline RL where $\pi_0$ represents a behavior policy~\citep{wu2019behavior, zhang2024implicit}, and language model  alignment~\citep{rafailov2024direct, yan2024efficient}. In the exploration problem, it may be of reasonable interest to use some (extrinsically defined) task's optimal policy as the prior. In this case, the agent will explore novel parts of the environment while staying ``close'' (in the sense of Kullback-Leibler divergence) to the prior policy, which may be an effective combination to satisfy the exploration-exploitation tradeoff. 
Though the regularized objective may be of independent interest, our primary objective is to obtain a solution to the maximum entropy exploration problem, as posed in Equation~\eqref{eq:entropy-objective}, without the additional regularization term. As such, maintaining such a regularization will bias the solution away from the pure MaxEnt solution.

In order to directly use the tilted matrix to determine the solution of the un-regularized problem, we need to take the limit $\beta \to \infty$.
However, instead of increasing the inverse temperature $\beta$, one can alternatively anneal in the policy space, by using an adaptive prior policy $\pi_0$, as introduced by~\cite{rawlik2012stochastic, rawlik-thesis}. Adopting this approach, we iterate the prior and posterior policies, by updating $\pi_0$ with the optimal policy resulting from the left eigenvector ($\pi^* \propto\pi_0u$). This left eigenvector, for a given $\pi_0$, is found by solving for the fixed point in Equation~\eqref{eq:u-update}. Intuitively, this method reduces the impact of the relative entropy regularization term, not by adjusting its weight $\beta$ but instead by adjusting the reference policy. Upon convergence, the prior and optimal policies will be the same, thereby incurring no additional entropic costs and only maximizing the reward function of interest (Eq.~\eqref{eq:reward-defn}). This method is known as \textit{posterior policy iteration}, or PPI~\cite{rawlik2012stochastic}, and was also studied in the deep RL setting by~\cite{adamczyk2025eval}. In Appendix~\ref{app:ppi} we prove the convergence of PPI in this setting, where the reward function also changes.
We provide the corresponding pseudocode for EVE below.

\begin{algorithm}[H]
\caption{EVE}
\label{alg:entropy-maximization}
\begin{algorithmic}[1]
\Require Transition matrix $P(s'|s,a)$, initial prior policy $\pi_0(a|s)$
\Require Inverse temperature $\beta(t) \geq 1$, maximum $u$ iterations $N$, maximum PPI iterations $T$

\For{$t = 1$ \textbf{to} $T$}
    \For{$n = 1$ \textbf{to} $N$}
    \State $u \leftarrow \mathcal{T}\left( u ; \beta(t) \right)$ \quad \texttt{// Update according to Equation~\eqref{eq:u-update}}
    \EndFor
    \State Update optimal policy: $\pi^*(a|s) = \frac{\pi_0(a|s) u(s,a)}{\sum_a\pi_0(a|s) u(s,a)}$
    \State $\pi_{0}(a|s) \gets \pi^*(a|s)$ \quad \texttt{// Update prior policy}
\EndFor
\newline 
\noindent
\textbf{Output:} Exploration policy, $\pi^*$
\end{algorithmic}
\end{algorithm}

\section{Experiments}

We use tabular GridWorld environments to verify our theoretical results and test the proposed method (given in Algorithm~\ref{alg:entropy-maximization}). Specifically, we use deterministic dynamics with four discrete actions (in the cardinal directions). The initial state is denoted by a green circle and absorbing states are denoted by the grey region (transitions the agent back to the initial state). Although we use a tabular approach (with transition matrix given), the agent can alternatively learn a model~\citep{sutton1991dyna, hafner2019learning, moerland2023model} during training. Notably, because of the form of the update in Equation~\eqref{eq:u-update}, one must also model the backward transitions, or the ``parent'' state-action pairs.

As baselines, we consider the MaxEnt algorithm from~\cite{hazan2019provably} and a set of rollout-based techniques where a policy's steady-state distribution is calculated, and a new reward function is defined according to $r(s,a)=-\log d_\pi(s,a)$. Experimentally, we observe oscillatory behaviors in the rollout-based methods, as outlined in~\citep{lee2019efficient}. To mitigate this effect, we introduce a learning rate (finetuned for each baseline) to slowly mix the consecutive reward functions. We also warmstart the RL loop with the previous value function. We found both of these steps were necessary for stable convergence of the baselines. We use a discounted approach, where the reward function from the current stage is optimized with discounted soft Q learning ($\gamma \in \{0.8, 0.9, 0.95, 0.99\}$). We also use an average-reward differential soft Q learning algorithm to isolate the effects of discounting and our own update equation. The soft Q learning baselines use a fixed number of steps ($50$) before the reward function is updated and a linear schedule for $\beta\in\{1,2,\dotsc,10\}$. On the horizontal axis, ``iterations'' refers to the number of synchronous updates to the value function (i.e. value iteration steps). 

Importantly, EVE does not require tracking a distribution or reward function, since these are built into the update equation, by design. As such, there are no oscillatory effects, so the corresponding learning rate can be removed altogether. The MaxEnt algorithm uses a convex combination of greedy policies to form a stochastic optimal policy, whereas the EVE policy is naturally stochastic at each iteration. As such, we circumvent the potentially large memory footprint of MaxEnt, which requires storing all previous policies.

\begin{figure}
    \centering
    \includegraphics[width=0.9\linewidth]{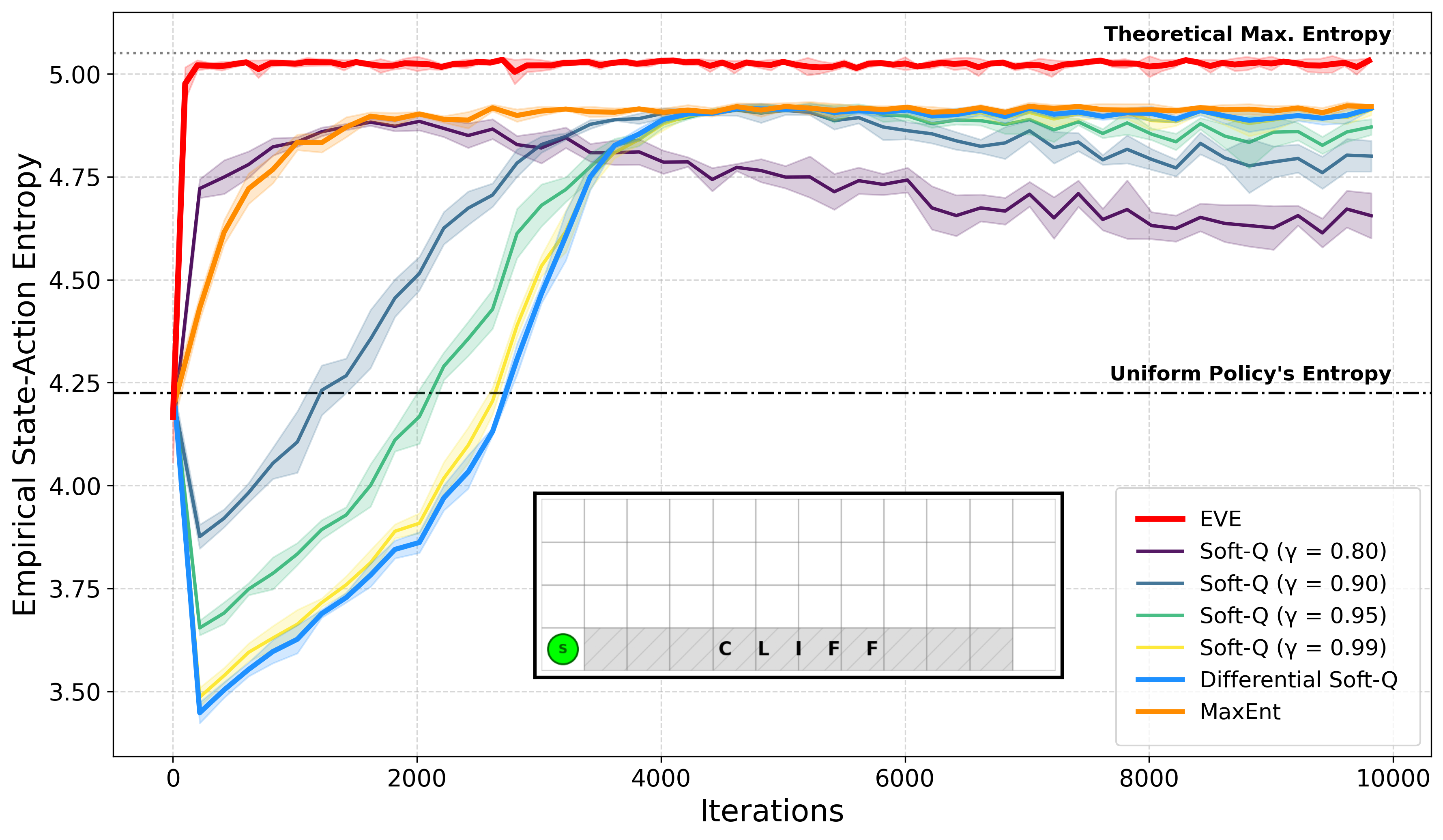}
    \caption{EVE converges to an exploration policy that achieves maximum entropy. Compared to the baselines, the optimal policy found by EVE produces a higher entropy and converges much faster. (Inset) ``CliffWorld'' environment used. The green circle denotes the initial state; stepping into the cliff resets the agent. Each line represents the mean over 5 independent initializations and the shaded region denotes one standard deviation.}
    \label{fig:cliff}
\end{figure}

Our experiments in Figure~\ref{fig:cliff} show that EVE is able to find policies capable of efficient exploration: nearly obtaining the maximum possible entropy, $\log |\mathcal{S}||\mathcal{A}|$. We note that in many environments it is not possible to attain this value for any policy due to the constraints imposed by the transition dynamics. 
To induce a uniform distribution, EVE's resulting policy is highly non-uniform, as the agent must navigate away from the cliff to continue exploring. The policies learned by rollout methods lead to lower entropy (worse coverage) and take more iterations to obtain. The code to reproduce these experiments can be found at \url{https://github.com/JacobHA/EVE}.

\section{Related Work}
Being a fundamental pillar of reinforcement learning, the issue of exploration has cemented itself as a cornerstone of theoretical and experimental RL literature, with ever growing interest, especially as environments become more complex and rewards more sparse. A prominent line of work formulates exploration as the maximization of entropy over state or state-action visitation distributions. For instance, \cite{hazan2019provably} demonstrate that maximizing the entropy of discounted occupancy measures can be solved as a convex program with polynomial sample complexity, marking an important step toward provably convergent exploration. Similarly, \cite{jain2023maximum} explore the maximum entropy problem, but with non-Markovian policies within a discounted framework. The reliance on discounting in these approaches inherently skews the visitation measure away from the true long-run steady-state. Consequently, they are unable to leverage the exact spectral connection to the stationary distribution's eigenvectors that our undiscounted, average-reward formulation natively exploits. Table 1 of \cite{mutti2022importance} provides a comprehensive overview of exploration methods.

Spectral methods have also been fruitfully applied to RL exploration in other contexts. \cite{machado2017eigenoption} utilize the eigenvectors of the successor representation to discover ``eigenoptions'' that encourage diverse exploratory behaviors. While conceptually related, EVE differs fundamentally in both objective and mechanism: rather than learning options to implicitly aid exploration, we focus on explicitly optimizing Equation~\eqref{eq:entropy-objective}, deriving a closed-form update from the eigenvectors of the tilted transition matrix. 

Finally, the precise formulation of the entropy objective and its optimization landscape separates our work from previous approaches. While much of the literature focuses strictly on state entropy, recent studies emphasize the importance of maximizing entropy across the joint state-action space \citep{zhang2021exploration}. Whereas \cite{zhang2021exploration} tackles this using R\'enyi entropy within a total-reward framework, EVE employs Shannon entropy within the average-reward setting. Overall, many of these experimental methods use function approximation techniques and thus could offer a viable pathway for extending EVE to continuous, model-free problems.

\section{Discussion}
EVE offers a principled, computationally efficient solution to the exploration problem by directly optimizing the steady‐state entropy of state‐action visits while avoiding costly rollouts or discounted surrogates. 
By leveraging recent analytical advances in entropy-regularized average-reward RL, we derived a self‐consistent update mechanism for solving the MaxEnt exploration problem. Because our formulation is naturally embedded in the entropy-regularized setting, we can leverage the flexibility of the prior policy $\pi_0$, obtaining a smooth transition from a regularized exploration policy to the un-regularized maximum-entropy solution via PPI.  

Beyond the new perspective and theoretical techniques provided for the pure exploration problem, EVE can also serve as a powerful pretraining objective~\cite{liu2021behavior}. An agent initialized with this maximum entropy policy can uniformly cover the state-action space, which is especially valuable in sparse-reward environments. We believe that EVE has the potential to become a versatile building block for both reward-free data collection and downstream RL fine-tuning, thanks to its broad coverage and efficient learning.

\subsection{Limitations and Future Work}
The solution method proposed is designed to solve the maximum entropy problem. As such, it is not well-suited for ``noisy TV'' problems~\citep{burda2018exploration} which may instead require a more information-theoretic perspective on the intrinsic reward function. The current solution method only applies for the case of deterministic dynamics, based on the results in~\citep{PRR}; however an extension to stochastic dynamics exists and requires an additional loop over ``biasing functions''~\citep{UAI}. 

We also note that objectives beyond the Shannon entropy may be of interest~\citep{hazan2019provably}. When these alternate objectives depend on the stationary distribution, our results may also apply. For example, the marginal-matching problem~\cite{lee2019efficient} and state-entropy can also be readily addressed within this framework.
We envision extensions of these ideas in model-based RL, where a learned (backward) dynamics model is used to estimate the EVE update equation.
\section*{Acknowledgements}
RVK and JA acknowledge funding support from the NSF through Award No. PHY-2425180. This work is supported by the National Science Foundation under Cooperative Agreement PHY-2019786 (The NSF AI Institute for Artificial Intelligence and Fundamental Interactions, {\url{https://iaifi.org/}}). 


\bibliography{sample}
\bibliographystyle{rlj}

\clearpage
\appendix
\section{Proofs}\label{app:proofs}
\subsection{Derivation of Equation~\eqref{eq:self-consistent}}

Recall the definition of the tilted matrix, $\widetilde{P}$, from Equation~\eqref{eq:ptilda}. We substitute in the self-consistent reward function $r(s,a)=-\log u(s,a) v(s,a)$, where $u$ and $v$ are the left and right eigenvectors of $\widetilde{P}$, respectively.

\begin{equation*}
    \lambda u(s,a) = \sum_{s',a'} u(s',a')\frac{P^{\pi_0}(s',a'|s,a)}{(v(s,a)u(s,a))^\beta}\quad \text{and}\quad 
    \lambda v(s',a') = \sum_{s,a} v(s,a)\frac{P^{\pi_0}(s',a'|s,a)}{(v(s,a)u(s,a))^\beta},
\end{equation*}
which can be written as:
\begin{align}
    \sum_{s',a'} u(s',a')P^{\pi_0}(s',a'|s,a) &= \lambda u^{1+\beta}(s,a) v^\beta(s,a)\label{eq:app-evec1}\\
    \sum_{s,a} P^{\pi_0}(s',a'|s,a)u^{-\beta}(s,a)v^{1-\beta}(s,a)&=\lambda v(s',a').\label{eq:app-evec2}
\end{align}

\subsection{Derivation of Equation~\eqref{eq:u-update}}

For ease of notation, let $i, j, \text{and } k$ be vector indices which correspond to the sequential state-action pairs $(\bar{s}, \bar{a})$, $(s,a)$, and $(s',a')$, respectively. Accordingly, let $P_{ji}$ denote $P^{\pi_0}(s,a|\bar{s}, \bar{a})$, and let $P_{kj}$ denote $P^{\pi_0}(s',a'|s,a)$. Equations~\eqref{eq:app-evec1} and ~\eqref{eq:app-evec2} are written more compactly as:
\begin{align}
    &\sum_k u_k P_{kj} = \lambda u^{1+\beta}_j v^\beta_j \label{eq:app-evec3}\\
    &\sum_i P_{ji}u_i^{-\beta}v_i^{1-\beta}=\lambda v_j\label{eq:app-evec4},
\end{align}
where for Equation~\eqref{eq:app-evec3} we have also shifted the timestep to align with the right eigenvector in Equation~\eqref{eq:app-evec4}. Now, by solving for the right eigenvector in terms of the left, we have, for the two state-action pairs,
\begin{equation}
    v_j = \left(\frac{1}{\lambda} \frac{(u^T P)_j}{ u_j^{\beta+1}}\right)^{\frac{1}{\beta}} \quad \textrm{and} \quad v_i^{1-\beta}= \left(\frac{1}{\lambda} \frac{(u^T P)_i}{ u_i^{\beta+1}}\right)^{\frac{1-\beta}{\beta}}\label{vj}.
\end{equation}

Substituting these expressions into Equation~\eqref{eq:app-evec4}, we have
\begin{equation}
    \sum_i P_{ji}u_i^{-\beta}\left(\frac{1}{\lambda} \frac{(u^T P)_i}{ u_i^{\beta+1}}\right)^{\frac{1-\beta}{\beta}}=\lambda \left(\frac{1}{\lambda} \frac{(u^T P)_j}{ u_j^{\beta+1}}\right)^{\frac{1}{\beta}}
\end{equation}
The eigenvalue cancels, and we can solve for $u_j=u(s,a)$ as:

\begin{equation}
    u_j^{(1+\beta)/\beta} = \frac{\left(\sum_k u_k P_{kj}\right)^{1/\beta}}{\sum_i P_{ji} u_i^{-1/\beta} (u^T P)_i^{(1-\beta)/\beta}}
\end{equation}
or, written in terms of state-action pairs, for $\beta=1$,
\begin{equation}
    u(s,a) = \sqrt{ \frac{ \sum_{s',a'} u(s',a') P(s',a'|s,a)}{\sum_{\bar{s},\bar{a}} u^{-1}(\bar{s},\bar{a}) P(s,a|\bar{s},\bar{a}) }}.
\end{equation}


\subsection{Convergence of EVE} \label{}
\begin{lemma}\label{lem:primitivity}
Let the prior policy during posterior policy iteration (PPI), be denoted $\pi_0(t)$. Then, the index of primitivity, $m(t)=m$, for the matrix $P^{\pi_0(t)}$ is a constant.
\end{lemma}
Notably, this result also holds when a softened version of PPI is used: i.e. if a rolling average of policies is used rather than ``hard updates'', the result remains, because the primitivity is a topological property of the connectivity of the underlying state-action graph. The value of $m$ is based only on the minimum number of steps for which there is a positive probability of transitioning between any two state-action pairs. Since iterations of PPI yield policies with the same support (i.e. absolute continuity holds: $\pi_0(t+1) \ll \pi_0(t)$), various policies cannot affect the number of steps required, only the corresponding probabilities.

\begin{theorem*}
    \EVEConvergence
\end{theorem*}

\begin{proof}
We follow the techniques discussed in~\cite{lemmens2014birkhoff}.
First we show that $\mathcal{T}$ is positive, homogeneous, and monotone for all $\beta>1$. We focus on the case of ``$m$-step returns'', where the one-step transition matrix $P^{\pi_0}$ is replaced by $M\doteq (P^{\pi_0})^m$, where $m$ is the smallest integer (number of steps in the MDP) such that $M>0$; ensuring as discussed above, that all state-action pairs are reachable from one another. 
\paragraph{Positive}
Positivity follows trivially since $M>0$: Both the numerator and denominator in Equation~\eqref{eq:u-update} are strictly positive for any $u > 0$ and all real values of $\beta$.

\paragraph{Homogeneous}
Let $\mathcal{T}(u)$ denote the operator defined by Equation~\eqref{eq:u-update}. To show that the operator is homogeneous, we evaluate $\mathcal{T}(cu)$ for a constant $c > 0$:
\begin{align*}
[\mathcal{T}(cu)]_j &= \left( \frac{ (M(cu))_j^{\frac{1}{\beta}}}{\sum_i M_{ji} (cu)_i^{-\frac{1}{\beta}} (M(cu))_i^{\frac{1-\beta}{\beta}}} \right)^{\frac{\beta}{1+\beta}}
\end{align*}
Factoring the scalar $c$ out of the linear matrix multiplications yields:
\begin{align*}
[\mathcal{T}(cu)]_j &= \left( \frac{ c^{\frac{1}{\beta}} (Mu)_j^{\frac{1}{\beta}}}{\sum_i M_{ji} c^{-\frac{1}{\beta}} u_i^{-\frac{1}{\beta}} c^{\frac{1-\beta}{\beta}} (Mu)_i^{\frac{1-\beta}{\beta}}} \right)^{\frac{\beta}{1+\beta}}
\end{align*}
Next, we consolidate the scalar $c$ in the denominator by combining its exponents throughout:
\begin{align*}
[\mathcal{T}(cu)]_j &= \left( c^{\frac{1+\beta}{\beta}} \right)^{\frac{\beta}{1+\beta}} [\mathcal{T}(u)]_j = c [\mathcal{T}(u)]_j
\end{align*}

Thus, the fixed-point operator $\mathcal{T}$ is homogeneous of degree $1$ for all $\beta > 0$.
\paragraph{Monotonicity}
A sufficient condition for the monotonicity of $\mathcal{T}$ requires $\frac{\partial [\mathcal{T}(u)]_k}{\partial u_i} \ge 0$ for all $u > 0$ and $M \ge 0$. Let $[\mathcal{T}(u)]_k = \left( \frac{N_k(u)}{D_k(u)} \right)^\delta$, where:
\begin{align*}
\delta &= \frac{\beta}{1+\beta}, \quad N_k(u) = (Mu)_k^{\frac{1}{\beta}}, \quad D_k(u) = \sum_j M_{kj} u_j^{-\frac{1}{\beta}} (Mu)_j^{\frac{1-\beta}{\beta}}.
\end{align*}
Since $\delta > 0$ for all $\beta > 0$, it is sufficient to show that $\frac{\partial N_k(u)}{\partial u_i} \ge 0$ and $\frac{\partial D_k(u)}{\partial u_i} \le 0$. For the numerator, applying the chain rule gives:
\begin{align*}
\frac{\partial N_k(u)}{\partial u_i} &= \frac{1}{\beta} (Mu)_k^{\frac{1-\beta}{\beta}} M_{ki} \ge 0.
\end{align*}

For $D(u)$, applying the product and chain rules yields:
\begin{align*}
\frac{\partial D_k(u)}{\partial u_i} &= \sum_j M_{kj} \left[ \left(-\frac{1}{\beta}\right) u_j^{-\frac{1+\beta}{\beta}} (Mu)_j^{\frac{1-\beta}{\beta}} \delta_{ji} + u_j^{-\frac{1}{\beta}} \left(\frac{1-\beta}{\beta}\right) (Mu)_j^{\frac{1-2\beta}{\beta}} M_{ji} \right].
\end{align*}

Since this term is in the denominator, the derivative must be non-positive for any $M > 0$. Hence, the scalar coefficients of both terms inside the sum must be non-positive, $\beta \geq 1$.

\paragraph{Contraction}
To analyze the convergence of $\mathcal{T}$, we operate in the projective space of strictly positive vectors $\mathbb{R}_{>0}^n$.
The \textit{projective metric} (also ``Hilbert's projective metric'') is defined for any $x, y \in \mathbb{R}_{>0}^n$ as:
\begin{equation}
    d_H(x,y) \doteq \ln \left( \max_{i} \frac{x_i}{y_i} \right) - \ln \left( \min_{i} \frac{x_i}{y_i} \right).
\end{equation}

This metric is naturally scale-invariant, meaning $d_H(cx, y) = d_H(x,y)$ for any scalar $c > 0$, making it an appropriate metric for the average-reward setting where the differential value function is well-defined up to constant shifts (multiplicative constants in the exponential space that we operate under). We rely on three standard properties of $d_H$ for element-wise operations:
\begin{enumerate}
    \item \textbf{Exponentiation:} $d_H(x^p, y^p) = |p| d_H(x, y)$ for any $p \in \mathbb{R}$.
    \item \textbf{Multiplication/Division:} $d_H(x \circ u, y \circ w) \le d_H(x, y) + d_H(u, w)$, where $\circ$ is the Hadamard (element-wise) product.
    \item \textbf{Linear Contraction (Birkhoff-Hopf Theorem):} For any strictly positive matrix $M > 0$, the linear map $x \to Mx$ acts as a strict contraction: $d_H(Mx, My) \le \tau d_H(x,y)$, where the Birkhoff contraction coefficient is $\tau = \tanh\left(\frac{\Delta(M)}{4}\right) < 1$, and $\Delta(M)$ is the projective diameter of $M$.
\end{enumerate}
Further discussion and details of this metric can be found in, e.g., ~\cite{kohlberg1982contraction}.

As established in prior sections, $\mathcal{T}$ is positive, homogeneous of degree 1, and monotonically order-preserving for $\beta \ge 1$ and thus satisfies the hypotheses of Birkhoff's convergence theorem. We can track the Lipschitz constant of $\mathcal{T}$ under $d_H$. 

We decompose the operator $\mathcal{T}(u)$ into its numerator $N(u)$, denominator $D(u)$, and outer exponent $\delta = \frac{\beta}{1+\beta}$, evaluating the metric contraction rate of each part sequentially. Let $x, y \in \mathbb{R}_{>0}^n$ be two positive vectors with initial projective distance $\Delta = d_H(x, y)$.

First, for the numerator $N(u) = (Mu)^{\frac{1}{\beta}}$, we use the rules for positive operators and exponentiation written above:
\begin{align*}
d_H(N(x), N(y)) &\le \frac{1}{\beta} d_H(Mx, My) \le \frac{\tau}{\beta} \Delta.
\end{align*}

Second, we evaluate the intermediate vector $f(u)$ inside the denominator's sum, defined element-wise as $f_j(u) = u_j^{-\frac{1}{\beta}} (Mu)_j^{\frac{1-\beta}{\beta}}$. Using the triangle inequality for element-wise multiplication, the scaling property of exponentiation and simplifying:
\begin{align*}
d_H(f(x), f(y)) &\le \left| -\frac{1}{\beta} \right| d_H(x, y) + \left| \frac{1-\beta}{\beta} \right| d_H(Mx, My).\\
d_H(f(x), f(y)) &\le \frac{1}{\beta} \Delta + \frac{\beta-1}{\beta} \tau \Delta = \left( \frac{1 + \tau(\beta-1)}{\beta} \right) \Delta.
\end{align*}

The denominator $D(u) =M f(u)$ applies the linear operator $M$ to $f(u)$, contracting the distance by another factor of $\tau$:
\begin{align*}
d_H(D(x), D(y)) &\le \tau d_H(f(x), f(y)) \le \tau \left( \frac{1 + \tau(\beta-1)}{\beta} \right) \Delta.
\end{align*}

Combining the numerator and denominator via element-wise division (cf. the Multiplication/Division rule above), the ratio $R(u) = \frac{N(u)}{D(u)}$ yields:
\begin{align*}
d_H(R(x), R(y)) &\le d_H(N(x), N(y)) + d_H(D(x), D(y)) \\
&\le \left[ \frac{\tau}{\beta} + \frac{\tau}{\beta} + \frac{\tau^2(\beta-1)}{\beta} \right] \Delta = \left( \frac{2\tau + \tau^2(\beta-1)}{\beta} \right) \Delta.
\end{align*}

Finally, applying the outer exponent $\delta = \frac{\beta}{1+\beta}$ gives the global contraction rate $C(\beta, \tau)$ for the full operator $\mathcal{T}$:
\begin{align*}
d_H(\mathcal{T}(x), \mathcal{T}(y)) &\le \frac{\beta}{1+\beta} \left( \frac{2\tau + \tau^2(\beta-1)}{\beta} \right) \Delta \\
&= \frac{2\tau + \tau^2(\beta-1)}{1+\beta} \Delta,
\end{align*}
where the coefficient $C(\beta, \tau)=\frac{2\tau + \tau^2(\beta-1)}{1+\beta} < 1$ since $\tau < 1$. 

Therefore, by the Banach fixed-point theorem, the sequence of projective distances $d_H(\mathcal{T}^{km}(x), \mathcal{T}^{km}(y)) \le C(\beta, \tau)^k \Delta$ vanishes as $k \to \infty$, proving convergence to a unique positive projective fixed point $u^*$. Because the index of primitivity $m$ remains constant across different prior policies during Posterior Policy Iteration (cf. Lemma~\ref{lem:primitivity}), the effective contraction rate per single fixed-point iteration of $\mathcal{T}$ is bounded by $C(\beta, \tau)^{1/m}$. Thus, the algorithm converges linearly under the Hilbert projective metric.

\end{proof}
\subsection{Posterior Policy Iteration Proof}\label{app:ppi}
Here we show that PPI iterations lead to a monotonically increasing entropy-rate.
\begin{proof}
Consider, at the $n^\textrm{th}$ iteration of PPI, the prior policy $\pi_n$ and corresponding posterior policy $\pi_{n+1}\propto \pi_n(a|s) u(s,a)$. Denote the entropy-regularized objective as $J(\pi, \pi_n)=\mathbb{E}_\pi (r - \beta^{-1} D_{\textrm{KL}}\left(\pi \,\|\, \pi_n \right) )$. Clearly $J(\pi_n,\pi_n) = \rho(\pi_n)$ (average-reward, i.e. entropy-rate, for policy $\pi_n$).
Now $J(\pi_{n+1}, \pi_n) \geq J(\pi_n, \pi_n) = \rho(\pi_n)$ by definition, since $\pi_{n+1}$ is the maximizer of the objective.
Since the Kullback-Liebler divergence is non-negative, we also have $\rho(\pi_{n+1}) \geq J(\pi_{n+1}, \pi_n)$.
Combining the two inequalities, we have $\rho(\pi_{n+1}) \geq \rho(\pi_n)$: in other words, the entropy-rate is monotonically increasing as a function of $n$. At both steps, equality is obtained only if {$\pi_n=\pi_{n+1}$} (at convergence).
\end{proof}


\end{document}